\documentclass{article}

\usepackage{microtype}
\usepackage{graphicx}
\usepackage{subfigure}
\usepackage{booktabs} 
\usepackage[a4paper, total={6.5in, 9in}]{geometry}
\usepackage{hyperref}
\usepackage{algorithm,algorithmic}

\usepackage{multirow}
\usepackage{authblk}

\usepackage{amsmath}
\usepackage{amssymb}
\usepackage{mathtools}
\usepackage{amsthm}

\usepackage[capitalize,noabbrev]{cleveref}
\def\zetab{\boldsymbol{\zeta}}
\def\Kcal{\mathcal{K}}
\def\Scal{\mathcal{S}}
\def\w{\mathbf{w}}
\def\Vcal{\mathcal{V}}

\def\Jcal{\mathcal{J}}
\def\Ical{\mathcal{I}}

\def\st{\ \ \mathrm{s.t.} \ \ }
\def\c{\mathbf{c}}
\def\g{\mathbf{g}}
\def\x{\mathbf{x}}

\def\z{\mathbf{z}}
\def\v{\mathbf{v}}

\def\cs{\mathbf{c}_{s}}

\def\cah{\hat{\mathbf{c}}^{\ast}_{a}}
\def\csh{\hat{\mathbf{c}}^{\ast}_{s}}
\def\cshx{\hat{\mathbf{c}}^{\ast,\tilde{x}}_{s}}
\def\cahx{\hat{\mathbf{c}}^{\ast,\tilde{x}}_{a}}
\def\cshxx{\hat{\mathbf{c}}^{\ast,\mathbf{x}'}_{s}}
\def\cahxx{\hat{\mathbf{c}}^{\ast,\mathbf{x}'}_{a}}
\def\cstxx{\tilde{\mathbf{c}}^{\ast,\mathbf{x}'}_{s}}
\def\catxx{\tilde{\mathbf{c}}^{\ast,\mathbf{x}'}_{a}}
\def\cstx{\tilde{\mathbf{c}}^{\ast,\tilde{x}}_{s}}
\def\catx{\tilde{\mathbf{c}}^{\ast,\tilde{x}}_{a}}
\def\caht{\hat{\mathbf{c}}_{a}}
\def\csht{\hat{\mathbf{c}}_{s}}
\def\cat{\tilde{\mathbf{c}}^{\ast}_{a}}
\def\cst{\tilde{\mathbf{c}}^{\ast}_{s}}
\def\cX{\mathcal{X}}

\def\o{\mathbf{o}}
\def\q{\mathbf{q}}
\def\cY{\mathcal{Y}}
\def\P{\mathbf{P}}
\def\b{\mathbf{b}}
\def\D{\mathbf{D}}
\def\W{\mathbf{W}}

\def\Da{\mathbf{D}_{a}}
\def\Ds{\mathbf{D}_{s}}

\def\Dcal{\mathcal{D}}

\def\ca{\mathbf{c}_{a}}
\def\Rb{\mathbb{R}}
\def\deltab{\boldsymbol{\delta}}

\def\Sb{\mathbb{S}}

\DeclareMathOperator*{\argmin}{arg\,min}
\DeclareMathOperator{\argmax}{arg\, max}

\newcommand{\norm}[2]{\left\| #1 \right\|_{#2}}
\newcommand{\myparagraph}[1]{\smallskip\noindent\textbf{#1.}}

\makeatletter
\newcommand{\printfnsymbol}[1]{%
  \textsuperscript{\@fnsymbol{#1}}%
}
\makeatother

\theoremstyle{plain}
\newtheorem{theorem}{Theorem}[section]
\newtheorem{proposition}[theorem]{Proposition}
\newtheorem{lemma}[theorem]{Lemma}

\theoremstyle{definition}
\newtheorem{definition}[theorem]{Definition}

\theoremstyle{remark}
\newtheorem{remark}[theorem]{Remark}

\usepackage[textsize=tiny]{todonotes}

\title{Reverse Engineering $\ell_p$ attacks: A block-sparse optimization approach with recovery guarantees}

\author[1]{Darshan Thaker \thanks{Equal Contribution.}}
\author[1]{Paris Giampouras \printfnsymbol{1}}
\author[1]{Ren\'e Vidal}

\affil[1]{Mathematical Institute of Data Science, Johns Hopkins University Baltimore, MD USA \authorcr
  \{\tt dbthaker, parisg, rvidal\}@jhu.edu}

\begin{document}

\date{\vspace{-6ex}}
\maketitle









\vskip 0.3in




\begin{abstract}
Deep neural network-based classifiers have been shown to be vulnerable to imperceptible perturbations to their input, such as $\ell_p$-bounded norm adversarial attacks. This has motivated the development of many defense methods, which are then broken by new attacks, and so on. This paper focuses on a different but related problem of reverse engineering adversarial attacks. Specifically, given an attacked signal, we study conditions under which one can determine the type of attack ($\ell_1$, $\ell_2$ or $\ell_\infty$) and recover the clean signal. We pose this problem as a block-sparse recovery problem, where both the signal and the attack are assumed to lie in a union of subspaces that includes one subspace per class and one subspace per attack type. We derive geometric conditions on the subspaces under which any attacked signal can be decomposed as the sum of a clean signal plus an attack. In addition, by determining the subspaces that contain the signal and the attack, we can also classify the signal and determine the attack type. Experiments on digit and face classification demonstrate the effectiveness of the proposed approach.
\end{abstract}

\section{Introduction}

Deep neural network based classifiers have been shown to be vulnerable to imperceptible perturbations to their inputs, which can cause the classifier to predict an erroneous output \cite{biggio2013evasion,szegedy_intriguing_2014}. Examples of such adversarial attacks include the Fast Gradient Sign Method (FGSM) \cite{Goodfellow:ICLR15} and the Projected Gradient Method (PGD) \cite{Madry:ICLR18}, which consist of small additive perturbations to the input that are bounded in $\ell_p$ norm and designed to maximize the loss of the classifier. In response to these attacks, many defense methods have been developed, including Randomized Smoothing \cite{Cohen:ICML19}. However, such defenses have been broken by new attacks, and so on, leading to a cat and mouse game between new attacks \cite{Athlaye:ICML18,Athalye2018ObfuscatedGG, Carlini:2017:AEE:3128572.3140444, Uesato2018AdversarialRA, Athalye2018OnTR} and new defenses \cite{madry2018towards, samangouei2018defensegan, Zhang2019TheoreticallyPT, 7546524, Kurakin2016AdversarialML, Miyato2017VirtualAT, Zheng2016ImprovingTR}. 

This paper focuses on the less well studied problem of \emph{reverse engineering adversarial attacks}. Specifically, given a corrupted signal $\x' = \x + \deltab$, where $\x$ is a ``clean'' signal and $\deltab$ is an $\ell_p$-norm bounded attack, the goal is to determine the attack type ($\ell_1$, $\ell_2$ or $\ell_\infty$) as well as the original signal~$\x$.

\myparagraph{Challenges}
In principle, this problem might seem impossible to solve since there could be many pairs $(\x,\deltab)$ that yield the same $\x'$. A key challenge is hence to derive conditions under which this problem is well posed. We propose to address this challenge by leveraging results from the sparse recovery literature, which show that one can perfectly recover a signal $\x$ from a corrupted version $\x' = \x + \deltab_0$ when both $\x$ and $\deltab_0$ are sparse in a meaningful basis. Specifically, it is shown in \cite{Wright:TIT10} that if $\x$ is sparse with respect to some signal dictionary $\Ds$, i.e., if $\x = \Ds\cs$ for a sparse vector $\cs$, and $\deltab_0$ is also sufficiently sparse, then the solution $(\c^*,\deltab^*)$ to the convex problem
\begin{equation}
\label{eq:cross_bouquet}
\min_{\c} \|\c\|_1 + \|\deltab \|_1 \quad \st \quad \x' = \Ds\c + \deltab 
\end{equation}
is such that $\c^* = \cs$ and $\deltab^* = \deltab_0$. In other words, one can perfectly recover the clean signal as $\x = \Ds\c^* = \Ds\cs$ and the corruption $\deltab_0$ by solving the convex problem in \eqref{eq:cross_bouquet}.



Unfortunately, these classical results from sparse recovery are not directly applicable to the problem of reverse engineering adversarial attacks due to several challenges:
\begin{enumerate}
\item An attack $\deltab$ may not be sparse. Indeed, $\deltab$ is usually assumed to be bounded in $\ell_p$ norm, where $p=1,2,\infty$. While results from sparse recovery can be extended to bounded $\ell_2$ errors, e.g. \cite{Candes-Romberg-Tao:CPAM06} considers the case where  $\deltab$ is $\ell_2$-bounded, such results only guarantee stable recovery, instead of exact recovery, of {\it sparse} vectors $\c_s$ close to $\c^*$. 
\item One of our goals is to determine the attack type. To do so, we need to exploit the fact that $\deltab$ is not an arbitrary vector, but rather a function of the attack type, the loss, the neural network and $\x$  (e.g., in the PGD method $\deltab$ is the projection of the gradient of the loss with respect to $\x$ onto the $\ell_2$ ball). The challenge is to devise an attack model that, despite these complex dependencies, is amenable to results from sparse recovery. In particular, we wish to impose structure on $\deltab$ that correlates its sparsity pattern to the attack type. 
\item Another goal is to correctly classify $\x'$ despite the attack $\deltab$. This is at odds with most sparse recovery results, which focus on reconstruction rather than classification. The main exceptions are the sparse and block-sparse representation classifiers \cite{Wright:PAMI09,Elhamifar:TSP12}, which divide the dictionary $\Ds$ into blocks corresponding to different classes and exploit the sparsity pattern of $\cs$ to determine the class of $\x$. But such classifiers are different from the neural network classifier given to us. 
\end{enumerate}

\myparagraph{Paper contributions}
This paper proposes a framework based on structured block-sparsity for addressing some of these challenges. Our key contributions are the following.

First, we develop a structured block-sparse model, a condition we show holds in a variety of settings, for decomposing attacked signals under three main assumptions about the signal and underlying network: 
\begin{enumerate}
\item We assume that the signal $\x'$ to be classified is the sum of a clean signal  $\x$ plus an $\ell_p$-norm bounded adversarial attack $\deltab$, i.e., $\x' = \x + \deltab$, i.e., additive attacks.

\item We assume that the clean signal $\x$ is \emph{block sparse} with respect to a dictionary of signals $\Ds$, i.e. $\x = \Ds \cs$, where $\Ds$ can be decomposed into multiple blocks, each one corresponding to one class,  and $\cs$ is {\it block-sparse} i.e. $\cs$ is only supported on a sparse number of blocks, but not necessarily sparse within those blocks.

\item We assume that the the $\ell_p$-norm bounded adversarial attack also admits a block-sparse representation in the columnspace of a dictionary $\Da$, which contains blocks corresponding to different $\ell_p$ bounded attacks.
\end{enumerate}

Second, we study conditions under which the aforementioned assumptions are feasible. In particular, we prove that $\ell_p$ attacks can be expressed as a structured block-sparse combination of other attacks for general loss functions when the attacked deep classifier satisfies some local linearity assumptions (e.g. ReLU networks).

Third, to determine the attack type and reconstruct the clean signal, we solve a convex optimization problem of the form:
		\begin{equation} \label{eq:generic_opt_form}
			\min_{\cs,\ca} ~~~ \|\cs\|_{1,2} + \|\ca\|_{1,2} ~~~ \st ~~~ \x' = \Ds\cs+ \Da\ca.
		\end{equation}
		Here, $\|\cdot\|_{1,2}$ is the $\ell_1/\ell_2$ norm that promotes structured block-sparsity on $\cs$ and $\ca$ exploiting the structure of $\Ds$ and $\Da$. For this optimization problem, we derive geometric data-dependent conditions under which the attack type and the clean signal can be recovered. These conditions rely on a special covering radius of $\Ds$ and $\Da$ and a generalization of angular distance induced by the $\ell_1 / \ell_2$ norm. 

Fourth, since solving \eqref{eq:generic_opt_form} can be computationally expensive due to the potentially large size of dictionaries $\Ds$ and $\Da$, we develop an efficient active set homotopy algorithm by first relaxing the constrained problem to a regularized problem instead and solving a sequence of subproblems restricted to certain blocks of $\Ds$ and $\Da$. 

Finally, we perform experiments on digit and face classification datasets to complement our theoretical results and demonstrate not only the robustness of block-sparse models on attacks arising from a union of $\ell_p$ perturbations, but also the effectiveness of our models in classifying the attack family.

\section{Related Work}

\myparagraph{Structured representations for data classification} 
Sparse representation of signals has achieved great success in applications such as image classification \cite{Yang:CVPR09,Mairal:CVPR08}, action recognition \cite{Yang:JAISE09,Castrodad:IJCV12}, and speech recognition \cite{Gemmeke:TASLP11,Sainath:TASLP11} (see \cite{Wright:IEEEProc10,Mairal:FT12} for more examples). These works rely on the assumption that data from a specific class lie in a low-dimensional subspace spanned by training samples of the same class. Hence, correct classification of amounts to  recovering the correct sparse representation of the signal on the columnspace of a certain dictionary. However, these works do not account for adversarially corrupted inputs, which pose significant challenges and are studied in this work.

\myparagraph{Structured representations for adversarial defenses} In the adversarial learning community, denoising-based defense strategies that leverage structured data representations have been recently proposed e.g. the work of \cite{samangouei2018defensegan} and \cite{moosavi2018divide} (see \cite{niu2020limitations} for a comprehensive survey). However, these approaches do not perform attack classification, and the key advantage of our approach is joint recovery of the signal and attack. To the best of our knowledge, our work is the first one to study this problem from a theoretical perspective. Additionally, even though our main goal is not to develop simply a stronger defense, we can compare the signal classification stage of our approach to defenses for a union of perturbation families simultaneously. Work such as \cite{tramer2019adversarial, maini2020adversarial, croce2019provable} develop adversarial training variants to tackle this problem. Our approach is distinct from adversarial training in that it requires no retraining of the neural network and can be applied post-hoc to adversarial examples.

\myparagraph{Detection of Adversarial Attacks} There is a vast literature on the detection of adversarial attacks, which work on the problem of detecting whether any example is an adversarial example or a clean example. These methods can be categorized into unsupervised and supervised methods, e.g. \cite{metzen2017detecting} and \cite{grosse2017statistical}. We refer the reader to \cite{bulusu2020anomalous} for a comprehensive survey on these methods. Our task is fundamentally different in that given an attacked image, we aim to classify the type of attack used to corrupt the image, and moreover provide theoretical guarantees under which this recovery is feasible. For this problem, \cite{maini2020perturbation} provide a method to classify attack perturbations similar to our work; however, they classify between $\ell_1, \ell_2$ attacks vs. $\ell_\infty$ attacks only. 

\section{Block-sparse model of $\ell_p$ attacked signals}
\label{sec:block-sparse-model}

Assume we are given a deep classifier $f_\theta : \cX \to \cY$, where $\cX$ is the input space, $\cY$ is the output space and $\theta$ are the classifier weights. Assume the classifier is trained using a loss function $L : \mathcal{Y}\times \mathcal{Y} \to \mathbb{R}^+$. Assume also an additive attack model $\x' = \x + \deltab$, where the attack $\deltab$ is a small perturbation to the input $\x$ that causes the classifier to make a wrong prediction, i.e., $f_\theta(\x')\neq f_\theta(\x)$.

We restrict our attention to $\ell_p$-norm bounded attacks, i.e. $\deltab\in\Delta_p=\{\deltab\in\Rb^n: \|\deltab\|_p\leq \epsilon \}$, which are crafted by finding a perturbation to $\x$ that maximizes the loss, i.e.:
\begin{equation} 
\max_{\deltab\in\Delta_p} L(f_\theta(\x+\deltab),y).
\end{equation}
Since solving this problem can be costly, a common practice is to maximize a first-order approximation of the loss. 
%
%
Letting $\g = \nabla_\x L(f_\theta(\x),y)$, we obtain the following gradient-based attacks for $p=1, 2, \infty$, respectively:
\begin{equation} 
\begin{split}
\deltab_1 = \epsilon ~ \mathbf{a},~~
\deltab_2 = \epsilon  ~ \frac{\g}{\| \g\|_2},~~
\deltab_\infty = \epsilon ~ \text{sign} (\g),
\end{split}
\end{equation}
where $\mathbf{a}$ denotes a unit norm vector where $\mathbf{a}_{i^\star} = \text{sign} ( g_{i^\star}) ~ \text{for} ~ i^\star := \argmax_i |g_i|$.
Note that $\ell_p$ attacks depend on the gradient of the loss with respect to the classifier input, the classifier output, the value of $p\geq 1$ used in the $\ell_p$-norm,
and the attack strength $\epsilon > 0$. 

\subsection{Validity of the block-sparse signal model}
We assume that the clean signal $\x$ (or features extracted from it) can be expressed in terms of a dictionary of signals $\Ds$ with coefficients $\cs$, i.e. $\x = \Ds \cs$. 
We also assume that $\Ds$ can be decomposed into $r$ blocks, one per class, and that its columns are unit norm.
Letting $\Ds[i]\in \Rb^{n\times m_i}$ denote the dictionary for the $i$th class and $\c_s[i]\in \Rb^{m_i}$ the corresponding set of coefficients, we can write the clean signal as
\begin{equation}
\x = \sum_{i=1}^r \Ds[i] \cs[i].
\end{equation} 

A priori this might seem like a strong assumption, which is violated by many datasets. However, we argue that the validity of this model depends on the choice of the dictionary (fixed or learned), the choice of additional structure on the coefficients (e.g, sparse, block-sparse), and the choice of a data embedding (e.g., fixed features such as SIFT, or unsupervised learned deep features). 

For example, as is common in image denoising, the dictionary $\Ds$ could be chosen as a Fourier or wavelet basis and the coefficients sparse with respect to such basis. Alternatively, as is common in face classification \cite{Belhumeur:IJCV98,Basri:PAMI03,Ho:CVPR03} where each class can be described by a low-dimensional subspace, $\Ds$ could be chosen as the training set and different blocks of the dictionary could correspond to different classes (subspaces). 
These results will motivate us to report experiments on face classification datasets, such as YaleB \cite{lee2005acquiring}, for which our modeling assumptions are satisfied. 

Even when the data from one class cannot be well approximated by a linear subspace, we note that the model $\x=\Ds\cs$ is actually nonlinear with (structured) sparsity constraints on $\cs$. Indeed, in manifold learning, data is often approximated locally by a subspace of nearest neighbors \cite{Roweis:Science00,Elhamifar:NIPS11}. When $\cs$ is block sparse, the model $\Ds\cs$ thus generates data on a manifold by stitching locally linear approximations. This will motivate our experiments on the MNIST dataset, where the set of all images of one digit is not a linear subspace, but our model still performs well.

\subsection{Structured block-sparse attack model}
We assume that the attack $\deltab\in\Delta_p$ can be expressed in terms of an attack dictionary $\Da$ with coefficients $\ca$, i.e., $\deltab = \Da \ca$. We also assume that the columns of $\Da$ are unit norm and that the dictionary can be decomposed into $a$ blocks, one per attack type. Moreover, we assume that the columns of $\Da$ are chosen as $\ell_p$ attacks evaluated at points in the training set. Therefore, each block of $\Da$ can be further subdivided into $r$ subblocks, one per class. As a consequence, the dictionary $\Da$ is composed of $r\times a$ blocks, $\Da[i][j]\in \Rb^{n\times k_{ij}}$, each one corresponding to data points from the $i$th class and $j$th attack type. Decomposing $\ca$ according to the block structure of $\Da$ so that $\ca[i][j]\in \Rb^{k_{ij}}$ is the vector of coefficients corresponding to block $\Da[i][j]$, we obtain the following threat model:
\begin{equation}
\label{eq:block-sparse-attack}
\deltab = \sum_{i=1}^r \sum_{j=1}^a \Da[i][j]\ca[i][j].
\end{equation}
Observe that when $\deltab$ is an $\ell_p$ attack evaluated at one of the points in the training set, the vector of coefficients $\ca$ is 1-sparse. In general, however, $\deltab$ will be evaluated at a test data point. In the next section, we will show in this case, we still expect $\deltab$ to be 1-block sparse for attacks on neural networks with ReLU activations. That is, we expect an attack of a certain type evaluated at a test point from one of the classes to be well approximated as a sparse linear combination of attacks of the same type but evaluated at other training data points from the same class.
%
\subsection{Validity of the attack model for ReLU networks}
Consider a ReLU network $f_\theta \colon \cX \to \Rb^r$, mapping the input to a point in $\Rb^r$, where $r$ is the number of classes.
The network is composed of $k$ layers, each consisting of an affine transformation followed by a ReLU non-linearity, i.e. 
\begin{equation}
f_\theta(\x) = \W_k (\ldots (\W_2 (\W_1 \x + \b_1)_+ + \b_2)_+ \ldots )_+ + \b_k,
\end{equation}
where $\theta=(\W_k, \ldots, \W_2, \W_1)$ denotes the parameters and $(\cdot)_+$ is the pointwise ReLU operation. The classification decision is then an argmax operation given by $\argmax_{i = 1 \ldots r} |z^{\x,k}_i|$, where $\z_\x^k = f_\theta(\x)$ is the network output. 

ReLU networks partition the input space into several polyhedral regions, inside each of which the network behaves like an affine map \cite{balestriero2020mad}.
Specifically, the affine region around $\x$ is given by the set of all points $\x'$ that produce the same \emph{sign pattern} as $\x$ after the ReLU activations at all the intermediate layers. 
More formally, defining $\zetab^l(\x) = {\rm sgn}(\W_l (\ldots (\W_2 (\W_1 \x + \b_1)_+ + \b_2)_+ \ldots ) + b_l)$ to be the sign pattern at layer $l$ for the input $\x$, the neural network $f_\theta$ behaves as an affine function
$f_\theta(\x) = \P_{S}^\top\x + \q$ in the region $S = \{ \x' \colon (\zetab^1(\x'), \zetab^2(\x'), \ldots, \zetab^k(\x')) = (\zetab^1(\x), \zetab^2(\x), \ldots, \zetab^k(\x))\}$. Therefore, the gradient of a loss $L(f_\theta(\x),y)$ in for all $\x'\in S$ 
is equal to
%
\begin{equation}
\begin{split}
\nabla_\x L(f_\theta(\x'), y) & = \frac{\partial f_\theta(\x')}{\partial \x'} \nabla_{\z_\x^k} L(\z_\x^k,y)  \\
 & = \P_S \nabla_{\z_x^k} L(\z_\x^k,y)
\end{split}
\label{eq:val_at}
\end{equation}%
where the last part of \eqref{eq:val_at} comes by the affine approximation of the output of the ReLU network i.e., $f_\theta(\x) = \P_{S}^\top\x + \q$  
Thus, the gradient of the loss function of a ReLU network at a test point in $S$ lives in the columnspace of a matrix $\P_S$. That being said, the gradient of the loss at a test point can be expressed as a linear combination of the gradients at training samples in the same region $S$. Note that this property holds true for popular loss functions e.g. cross-entropy loss, mean squared loss, etc. Moreover, we further assume that training samples in $S$ belong to the same class, which is a reasonable assumption to make for ReLU networks \cite{Sattelberg:2020}.

Hence, for any test point lying in region $S$, if there exists a training datapoint of the same class and also in $S$, then there will exist a vector $\ca$ that is a feasible solution of \eqref{eq:generic_opt_form} and block-sparse in the columnspace of $\Da$ with only one non-zero block (assuming that the $\deltab$ comes from a single attack from the family). 
Recent works \cite{lee2019towards} show that ReLU networks can be trained to have large linear regions, hence it is reasonable to expect that $S$ contains a training point.
\section{Reverse engineering of $\ell_p$-bounded attacks}
%
In Section \ref{sec:block-sparse-model} we introduced a block-sparse model of attacked signals, $\x' = \Ds\cs+ \Da\ca$, where $\Ds$ is a dictionary of clean signals (typically the training set), $\Da$ is a dictionary of attacks (typically $\ell_p$ attacks on the training set), and $\cs$ and $\ca$ are block-sparse vectors whose nonzero coefficients indicate the class and the attack type. In this section, we show how to reverse engineer the attack and clean signal.
\subsection{Block sparse optimization approach}
Assume that test sample $\x'$ has been corrupted by an attack of a single type. Since $\x'$ belongs to only one of the classes, we expect vectors $\cs$ and $\ca$ to be 1-block-sparse. Therefore, the problem of reverse engineering $\ell_p$ attacks can be cast as a standard block-sparse optimization problem, where we minimize the total number of nonzero blocks in $\cs$ and $\ca$ needed to generate $\x'$, i.e.
\begin{equation}
\begin{split}
    \min_{\cs,\ca} & \quad \sum^r_{i=1}\left(I(\|\cs[i]\|_2) + \sum^a_{j=1}I(\|\ca[i][j]\|_2) \right)   \\
    \st & \quad  \x'=\Ds\cs + \Da\ca ,
  \end{split}
  \label{eq:stand_blk_opt}
\end{equation}
where $I(\cdot)$ is the indicator function, i.e., $I(x)=1$ if $x\neq0$ and $I(x)=0$ if $x=0$. As is common in block-sparse recovery \cite{Elhamifar:TSP12}, a convex relaxation of the problem of minimizing the number of nonzero blocks is given by:
\begin{equation}
\begin{split}
    \min_{\cs,\ca} & \quad \sum^r_{i=1} \|\cs[i]\|_2 + \sum^a_{j=1} \|\ca[i][j]\|_2   \\
    \st & \quad  \x'=\Ds\cs + \Da\ca ,
  \end{split}
  \label{eq:stand_blk_opt_conv_relax}
\end{equation}
where the sum of the $\ell_2$ norms of the blocks, also known as the $\ell_1/\ell_2$ norm, is a convex surrogate for the number of nonzero blocks. 

\subsection{Active Set Homotopy Algorithm}
In practice, we further relax the problem and solve the regularized noisy version of problem
\eqref{eq:stand_blk_opt_conv_relax}, which can be written in the form,
\begin{equation}
\begin{split}
&\min_{\cs,\ca}\frac{1}{2}\|\x' - \Ds\cs - \Da\ca\|^2_2  \\
&+ \lambda_s \sum^r_{i=1} \|\cs[i]\|_2 + \lambda_a \sum^a_{j=1} \|\ca[i][j]\|_2 .
\end{split}
\label{eq:reg_objective}
\end{equation}

Because the size of the dictionaries $\Ds$ and $\Da$ can be large, it is crucial to develop scalable algorithms to solve the above optimization problem. Furthermore, the choice of $\lambda_s$ and $\lambda_a$ play a crucial role in enforcing the correct level of block-sparsity. High values of these parameters will drive the solutions $\cs, \ca$ to the zero vector, and too low values will result in a solution that is not block-sparse as desired. To address both issues, we develop an active-set based homotopy algorithm. The main insight is that instead of solving an optimization problem using the full data matrix, we can restrict $\Ds$ and $\Da$ to the blocks that correspond to non-zero blocks of the optimal $\cs$ and $\ca$. Using the optimality conditions of problem \eqref{eq:reg_objective}, we can derive an algorithm that maintains a list of block indices for both $\cs$ and $\ca$, denoted as the \textit{active sets} $T_s$ and $T_a$, and solve reduced subproblems based on these indices, significantly reducing runtime. The details of the derivation can be found in the Appendix.  

Additionally, to pick proper values of $\lambda_s$ and $\lambda_a$, we employ techniques from homotopy methods in sparse optimization to construct a sequence of decreasing values for $\lambda_s$ and $\lambda_a$ \cite{Malioutov:ICASSP05}. Traditionally, homotopy methods for $\ell_1$ minimization use the fact that the solution path as a function of regularization strength is piecewise linear, with breakpoints when the support of the solution changes. However, with the block-sparsity constraint, the path is nonlinear \cite{Yau:Stats2017}, and thus we approximate this path using a sequence of $\lambda$ values. The initial value of $\lambda_s$ and $\lambda_a$ is chosen to be a hyperparameter $\gamma \in (0, 1)$ times the value that produces the all-zeros vector based on the optimality conditions. In Algorithm~\ref{alg:activeset_homotopy} in the Appendix, we provide the details of the active set homotopy algorithm.
\section{Theoretical analysis of the block-sparse minimization problem}
\label{sec:theory}
In this section, we provide geometrically interpretable conditions under which the true signal class and attack type, which is generated by a single $\ell_p$ perturbation type, can be recovered from the nonzero blocks of $\cs,\ca$ using the proposed block-sparse minimization approach given in \eqref{eq:stand_blk_opt_conv_relax}. 

At first sight, one may think that existing conditions for block-sparse recovery in a union of low-dimensional subspaces \cite{Elhamifar:TSP12}, which require the subspaces to be disjoint and sufficiently separated, might be directly applicable to our problem. However, the adversarial setting presents several additional challenges. 
First, we do not need conditions for all block pairs, but only for the block pair formed by one signal subspace and one signal-attack subspace.
Second, the two non-zero blocks are not independent from each other, because if we determine the signal-attack block $(i^*,j^*)$, then we also determine the signal block $i^*$. 
Third, we need to disentangle not only one signal class from another, but also one attack type from another, and signals from attacks. 
%
%

In the following, we address these three challenges. Specifically, we significantly improve on the block-sparse recovery results of \cite{Elhamifar:TSP12} by getting rid of the strong assumption of disjointness among {\it all} pairs of subspaces spanned by the blocks of the dictionaries. Moreover, our analysis goes one step beyond previous efforts (e.g. \cite{Wang:PAMI2017}) to generalize the subspace-sparse recovery results \cite{You:ICML15,You:arxiv15-SSR} in the following ways. First, we focus on block-sparse recovery in a union of dictionaries as opposed to \cite{Wang:PAMI2017}, which focuses on a single dictionary. Second, our problem has more specific structure i.e., the dependency of non-zero blocks of the signal (see Remark 5.2). Third, our conditions are based on different newly introduced geometric measures i.e. covering radius and angular distances induced  by the $\ell_1/\ell_2$ norm. This leads to our first main result (Theorem~\ref{the:prc_cond}). Finally, we provide an additional result (see Theorem~\ref{the:drc_cond}), which relaxes Theorem~\ref{the:prc_cond} by involving the angular distances between points in a set of Lebesgue measure zero (instead of all points in the direct sum of the signal and attack subspaces) and a finite set.
Proposition \ref{propos} gives a necessary and sufficient condition 
for recovering the correct signal and attack by solving problem  \eqref{eq:stand_blk_opt_conv_relax}. Let $\Ical=\{1,2,\dots,r\}$ and $\Jcal = \{1,2,\dots,a\}$ denote the indices for the blocks of $\Ds$ and $\Da$ and vectors $\cs$ and $\ca$ respectively. We define the {\it correct-class minimum $\ell_{1}/\ell_{2}$ vectors} $\csh,\cah$ with non-zero blocks $\csh[i^\ast]$ and $\cah[i^{\ast}][j^{\ast}]$ as
\begin{equation}
\begin{split}
    \{\csh,\cah\} \equiv \argmin_{\cs,\ca}{\|\cs[i^\ast]\|_2+\|\ca[i^\ast][j^\ast]\|_2} \\ 
    \st \x' = \Ds[i^\ast]\cs[i^\ast] + \Da[i^\ast][j^\ast]\ca[i^\ast][j^\ast],
    \end{split}
    \label{eq:cor_minim_vec}
\end{equation}	
and the {\it wrong-class minimum $\ell_{1}/\ell_{2}$ norm vectors} $\cst,\cat$ as,
\begin{equation}
\begin{split}
  &  \{\cst,\cat \} \equiv \\
  &    \argmin_{\cs,\ca}\sum_{i\in \Ical \setminus \{i^{\ast}\} } \|\cs[i]\|_2  
    +\sum_{i\in \Ical , j\in \Jcal\setminus \{j^{\ast}\} }  \|\ca[i][j]\|_2   \\
  &  \st \\
  &  \x' = \sum_{i \in \Ical \setminus \{i^{\ast}\}}~\Ds[i]\cs[i]+\sum_{i\in \Ical, j\in \Jcal\setminus  \{j^{\ast}\} }~\Da[i][j]\ca[i][j] 
        \end{split}
\label{eq:wrong_minim_vec}
\end{equation}
Note that the non-zero blocks of $\cst$ and $\cat$ do not correspond to the correct signal and attack.
\begin{proposition}
The correct classes of the signal $\x \in \mathcal{S}^{\x}_{i^\ast}$ and the attack $\deltab\in \mathcal{S}^{\deltab}_{i^\ast, j^\ast}$, with $\mathcal{S}^\x_{i^\ast} \cap \mathcal{S}^{\deltab}_{i^\ast j^\ast} = \emptyset$,  can be recovered by solving \eqref{eq:stand_blk_opt_conv_relax} if and only if, $\forall i^\ast,j^\ast, \forall \x'\in (\mathcal{S}^\x_{i^\ast} \oplus \mathcal{S}^{\deltab}_{i^\ast, j^\ast}), \x\neq \mathbf{0}$,  the $\ell_{1}/\ell_{2}$ norm of the {\it correct-class  minimum $\ell_{1}/\ell_{2}$ vectors} $\csh,\cah$ is strictly less that of the {\it wrong-class minimum $\ell_{1}/\ell_{2}$ norm vectors} $\cst,\cat$, i.e.,
\begin{equation}
\begin{split}
\|\csh\|_{1,2} + \|\cah\|_{1,2} < \|\cst\|_{1,2} + \|\cat\|_{1,2}.
\end{split}
\label{prop:nec_suf_rec}
\end{equation}
%
\label{propos}
\end{proposition}
\begin{remark}
Note that disjointness between $\mathcal{S}^{\x}_{i^\ast}$ and $\mathcal{S}^{\deltab}_{i^\ast j^\ast}$ is necessary if we want to recover the class of the signal {\it and} the attack. However, in case that disjointness is violated, we can still guarantee the recovery of the correct class of the signal since we know that attacks depend on the signal class.
\end{remark}
Next, we aim to provide more geometrically interpretable conditions for the recovery of the signal and attack classes. First, we generalize the standard angular distance originally used for $\ell_{1}$ norm minimization problems, \cite{You:ICML15}, to the particular case of $\ell_{1}/\ell_{2}$ norm minimization.
\begin{definition}
Let $\Dcal$ be a set of unit $\ell_{2}$ norm columns of the dictionary $\D = [\D[1],\D[2],\dots,\D[c]]$ with $\D[i]\in\Rb^{n\times m}$. The angular distance between the atoms in $\Dcal$ and a vector $\v\in\Rb^{n}$ is defined as,
\begin{equation}
\theta_{1,2}( \v,\pm \Dcal) = \cos^{-1}\left(\frac{1}{\sqrt{m}}\|\D^\top\frac{\v}{\|\v\|_{2}}\|_{\infty,2}\right).
\label{def:ang_dist}
\end{equation}
\end{definition}
This definition can also be extended to a set of vectors $\Vcal$ as 
\begin{equation}
\theta_{1,2}( \Vcal,\pm\Dcal) = \inf_{\v\in \Vcal} \cos^{-1}\left(\frac{1}{\sqrt{m}}\|\D^\top\frac{\v}{\|\v\|_{2}}\|_{\infty,2}\right).
\label{def:ang_dist_gen}
\end{equation}
Similarly, we define a generalized version of the covering radius of a set induced by the $\ell_{1}/\ell_{2}$ norm.
\begin{definition}
The covering radius  $\gamma_{1,2}(\Dcal)$ of a set $\Dcal$ consisting of the columns of matrix $\D$ is defined as,
\begin{equation}
\gamma_{1,2}(\pm \Dcal) = \sup\{\theta_{1,2}(\v,\pm \Dcal), \v\in \Sb^{n-1}\cap \mathrm{span}(\Dcal) \}.
\end{equation}
\end{definition}
Note that the covering radius captures how well-separated the atoms of {\it the blocks} of $\D[i]$ are, and is a decreasing function of the distance between atoms.

Let $\Dcal_{i^{\ast}j^{\ast}}$ be the set that contains the columns of $\Ds[i^{\ast}],\Da[i^{\ast}][j^{\ast}]$, and  $\Dcal^{-}_{i^{\ast}j^{\ast}}$  the set with all remaining columns of the blocks $\Ds[i]$, $\forall  i \in \Ical \setminus i^{\ast}$ and  $\Da[i][j]$ for $i\in \Ical$ and $j\in \Jcal \setminus j^{\ast}$. 
We now provide a sufficient condition, which ensures the recovery of the correct classes of the signal and the attack.
\begin{theorem}
The correct classes of the signal and the attack of an adversarially perturbed signal in $\mathcal{S}_{i^{\ast}}\oplus \mathcal{S}_{i^{\ast} j^{\ast}}$ can be recovered by solving problem \eqref{eq:stand_blk_opt_conv_relax}, if the following primary recovery condition (PRC) %
\begin{equation}
\begin{split}
\gamma_{1,2}(\pm \Dcal_{i^{\ast}j^{\ast}}) 
< \theta_{1,2}(\mathcal{S}^\x_{i^\ast} \oplus \mathcal{S}^{\deltab}_{i^{\ast} j^{\ast}}, 
\pm \Dcal_{i^{\ast}j^{\ast}}^{-})
\end{split}
\label{eq:prc_cond}
\end{equation}
holds for the dictionaries $\Ds$ and $\Da$.
\label{the:prc_cond}
\end{theorem}
Theorem~\ref{the:prc_cond} offers a geometric intuition for recovery guarantees. Note that \eqref{eq:prc_cond} depends only on the properties of the dictionaries $\Ds$ and $\Da$. Specifically, \eqref{eq:prc_cond} is easier satisfied when a) the covering radius of $\Dcal_{i^{\ast}j^{\ast}}$  is small, meaning that columns of {\it both} $\Ds[i^{\ast}]$ and $\Da[i^{\ast}][j^{\ast}]$ are well distributed in $\mathcal{S}^\x_{i^\ast}$ and  $\mathcal{S}^{\deltab}_{i^{\ast} j^{\ast}}$ respectively or b) the atoms of the remaining blocks of $\Ds,\Da$ are sufficiently away from $\mathcal{S}_{i^{\ast}}\oplus \mathcal{S}_{i^{\ast} j^{\ast}}$.  

Next, we derive the dual recovery condition (DRC), which only needs to hold a subset of points in $\mathcal{S}_{i^{\ast}}\oplus \mathcal{S}_{i^{\ast} j^{\ast}}$ called as {\it dual points}. Before illustrating the DRC, we first define the polar set induced by the $\ell_{1,2}$ norm and the dual points.
\begin{definition}
The polar of the set $\Dcal$ containing the columns of matrix $\D$ induced by the $\ell_{1,2}$ norm is defined as
\begin{equation}
\mathcal{K}^{o}_{\ell_{1,2}}(\Dcal)= \{ \v \in \mathcal{R}(\D): \frac{1}{\sqrt{m}}\|\D^{\top}\v\|_{\infty,2}\leq 1 \}.
\label{def:relative_polar}
\end{equation}
\end{definition}
where $\mathcal{R}(\D)$ is the range of $\D$.
\begin{definition}
The set of dual points of matrix $\D$, denoted as $\mathcal{A}(\Dcal)$, is the set of extreme points of $\mathcal{K}^{o}_{\ell_{1,2}}(\Dcal)$, which is the polar set of $\Dcal$.
\end{definition} 
\begin{theorem}
The correct classes of the signal and the attack can be recovered by solving problem \eqref{eq:stand_blk_opt_conv_relax}, if the following dual recovery condition (DRC) is satisfied
\begin{equation}
\gamma_{1,2}(\Dcal_{i^{\ast}j^{\ast}}) <  \theta_{1,2}\left(\mathcal{A}\left(\Dcal_{i^{\ast}j^{\ast}}\right), \pm \Dcal_{i^{\ast}j^{\ast}}^{-}\right),
\end{equation}
\label{the:drc_cond}
\end{theorem}
Theorem \ref{the:drc_cond} requires the covering radius of $\Dcal_{i^{\ast}j^{\ast}}$ to be smaller than the minimum angular distance between the dual points of $\Dcal_{i^{\ast}j^{\ast}}$, which form a set of Lebesgue measure zero, and elements of the set $\Dcal_{i^{\ast}j^{\ast}}^{-}$.

\section{Experiments}

In this section, we present experiments on the Extended YaleB Face dataset and the MNIST dataset. 

\subsection{Experimental Setup}

\myparagraph{Network architectures and attack evaluation} For the YaleB Face dataset, we train a simple 3-layer fully-connected ReLU neural network with 256 hidden units per layer, which already serves as a strong baseline obtaining 96.3\% accuracy. For the MNIST dataset, we train a 4-layer convolutional network, identical to the architecture from \cite{carliniwagner}. 
All networks are trained with a cross-entropy loss. 

We consider the family of $\{\ell_1, \ell_2, \ell_\infty\}$ PGD attacks. $\ell_1$ PGD refers to the Sparse $\ell_1$ PGD attack from \cite{tramer2019adversarial}. For optimization, we use the active set homotopy algorithm developed in Section 4.2.  The Appendix contains full experimental details. 

\myparagraph{Metrics} We choose $\Ds$ as a dictionary whose columns are the flattened training images, and $\Da$ is a dictionary whose columns are the $\ell_p$ perturbations for each training image. For each block in $\Ds$ and $\Da$, we subsample $200$ training datapoints to limit the dictionary size, and we normalize the columns of the dictionary to unit $\ell_2$ norm to keep the same scaling for all blocks. For a given perturbed image $\x'$, we run Algorithm \ref{alg:activeset_homotopy} to obtain the output coefficients $\csht$ and $\caht$. We define the predicted block indices for the signal and attack dictionaries to be:
\begin{align}
    \label{eq:ihat} \hat{i} &= \argmin_i \| \x' - \Ds[i] \csht[i] - \Da \caht \|_2 \\
    \label{eq:jhat} \hat{j} &= \argmin_j \| \x' - \Ds \csht - \Da[\hat{i}][j] \caht[\hat{i}][j] \|_2
\end{align}
Using these indices, we define two classification methods and one attack detection method:
\begin{enumerate}
    \item \textit{SBSC (Structured Block-Sparse Classifier)}: This method predicts the class of the test image as $\hat{i}$.
    \item \textit{SBSC+CNN (Denoiser)}: From $\csht$, this method computes a denoised image as $\hat{\x}= \Ds[\hat{i}] \csht[\hat{i}]$ and then predicts the class of the test image from the output of the original network at the denoised datapoint, i.e. $f_{\theta}(\hat{\x})$ 
    \item \textit{SBSAD (Structured Block-Sparse Attack Detector)}: This method returns $\hat{j}$, which represents the predicted attack type of the test image.
\end{enumerate}

For each method, we report the accuracy of prediction with respect to the ground truth label. For the SBSC and SBSC+CNN methods, the label is the correct label of the test image, while for SBSAD, the label is the true $\ell_p$ perturbation type that was applied to the test image. As a naive block-sparse classifier baseline, we denote BSC as a block-sparse classifier which does not model the structure of the attack perturbation, but simply models $\x' = \Ds \cs$ \cite{Elhamifar:TSP12}. We also consider a BSC+CNN baseline, which predicts the class from  $f_{\theta}(\Ds[\hat{i}] \csht[\hat{i}])$ as above, except $\csht$ is obtained from the BSC problem.

\begin{table*}[h!]
\centering
\caption{Adversarial image and attack classification accuracy on YaleB dataset. BSC denotes the block-sparse classifier baseline, SBSC denotes the structured block-sparse signal classifier, SBSC+CNN denotes the denoised model, and SBSAD denotes the structured block-sparse attack detector. }
\begin{tabular}{c||ccccc||c}
 \toprule
  \textbf{Yale-B} & CNN & BSC & BSC+CNN & SBSC & SBSC+CNN & SBSAD \\
 \midrule
 $\ell_\infty$ PGD ($\epsilon = 0.02$) & 15.1\% & 79\% & 2\% & \textbf{97\%} & 93\% & 52\%  \\
 $\ell_2$ PGD ($\epsilon = 0.75$) & 4.2\% & 51\% & 2\% & \textbf{96\%} & 87\% & 76\%  \\
 $\ell_1$ PGD ($\epsilon = 15$) & 53.7\% & 81\% & 3\% & \textbf{96\%} & 93\% & 39\% \\
 \midrule
 Average & 24.3\% & 70.3\% & 2.3\% & \textbf{96.3\%} & 91\% & 55.7\% \\
 \bottomrule
\end{tabular}
\label{table:yaleb_def}
\end{table*}

\begin{table*}[h!]
\centering
\caption{Adversarial image and attack classification accuracy on digit classification of MNIST dataset. See above table for column descriptions. The clean accuracy represents the accuracy of the method with unperturbed test inputs.}
\label{table:mnist_def}
\resizebox{\textwidth}{!}{
\begin{tabular}{c||c@{\;\;}c@{\;\;}c@{\;\;}c@{\;\;}c@{\;\;}c@{\;\;}c||c@{\;\;}c@{\;\;}c||c} 
 \toprule
\textbf{MNIST} & CNN & $M_\infty$ & $M_2$ & $M_1$ & MAX & AVG & MSD & BSC & SBSC & SBSC+CNN & SBSAD  \\
 \midrule
 Clean accuracy & 98.99\% & 99.1\% & 99.2\% & 99.0\% & 98.6\% & 98.1\% & 98.3\% & 92\% & 94\% & 99\% & - \\
 $\ell_\infty$ PGD ($\epsilon = 0.3$) & 0.03\% & \textbf{90.3\%} & 0.4\% & 0.0\% & 51.0\% & 65.2\% & 62.7\%  & 54\% & 77.27\% & 76.83\% & 73.2\% \\
 $\ell_2$ PGD ($\epsilon = 2.0$) & 44.13\% & 68.8\% & 69.2\% & 38.7\% & 64.1\% & 67.9\% & 70.2\% & 76\% & \textbf{85.34\%} & 85.17\% & 46\% \\
 $\ell_1$ PGD ($\epsilon = 10.0$) & 41.98\% & 61.8\% & 51.1\% & 74.6\% & 61.2\% & 66.5\% & 70.4\% & 75\% & \textbf{85.97\%} & 85.85\% & 36.6\% \\
 \midrule
 Average & 28.71\% & 73.63\% & 40.23\% & 37.77\% & 58.66\% & 66.53\% & 67.76\% & 68.33\% & \textbf{82.82\%} & 82.61\% & 51.93\% \\
 \toprule
 \textbf{Unseen Attacks} & & & & & & & & & & \\
 \midrule
 $\ell_\infty$ MIM ($\epsilon = 0.3$) & 0.02\% & \textbf{92.3\%} & 11.2\% & 0.1\% & 70.7\% & 76.7\% & 71.0\%  & 59.5\% & 74.3\% & 74.2\% & 79.0\% \\
 $\ell_2$ C-W ($\epsilon = 2.0$) & 0\% & 79.6\% & 74.5\% & 44.8\% & 72.1\% & 72.4\% & 74.5\% & \textbf{89.1\%} & 87.1\% & 87.1\% & 60.4\% \\
 $\ell_2$ DDN ($\epsilon = 2.0$) & 0\% & 63.9\% & 70.5\% & 40.0\% & 62.5\% & 64.6\% & 69.5\% & \textbf{88.8\%} & 87.2\% & 87.1\% & 57.8\% \\
 \midrule
 Average & 0\% & 78.6\% & 52.06\% & 28.3\% & 68.43\% & 71.23\% & 71.66\% & 79.13\% & \textbf{82.86\%} & 82.8\% & 65.73\% \\
 \bottomrule
\end{tabular}}
\end{table*}

\subsection{YaleB Face Dataset}

We first evaluate our method on images from the Extended YaleB Face Dataset \cite{lee2005acquiring}, a 38-way classification task. While the adversarial learning literature does not usually evaluate attacks on this dataset, we choose it because it exhibits the self-expressiveness property. Indeed, face images of an individual under varying lighting conditions have been shown to lie in a low-dimensional subspace \cite{Belhumeur:IJCV98,Basri:PAMI03,Ho:CVPR03}. Our goal is to complement our theoretical recovery guarantees by demonstrating the effectiveness of our approach in determining the correct signal and attack type, which is illustrated in Table \ref{table:yaleb_def}. For all perturbation types, we observe the SBSC approach significantly improves upon the accuracy of the undefended model, indicating the successful decoupling of the signal and attack. One phenomenon we see is the remarkable robustness of block-sparse classifiers, even without attack modelling. The BSC baseline consistently improves the adversarial accuracy of the undefended model; however, the low BSC+CNN accuracy indicates that there is still significant noise in the data modelling. On the other hand, the SBSC and SBSC+CNN are able to improve over the BSC baseline by around 20\%, indicating that explicitly modelling the attack helps signal classification for both the block-sparse classifier as well as the original classification network. 
\subsection{MNIST}
Despite the simplicity of the MNIST dataset, networks trained on MNIST are still brittle to attacks that arise from a union of perturbations. Specifically, in \cite{maini2020adversarial}, the authors observe that most state-of-the-art adversarial training defenses for MNIST are only robust to one type of $\ell_p$ attack.  

\myparagraph{Baselines} While we emphasize that our approach is not primarily a defense, but rather a principled attack classification and signal decoupling algorithm, we can still compare our signal classification accuracy to a variety of state of the art methods for defending against a union of attacks. 
First, we consider classifiers $M_1$, $M_2$, $M_\infty$ trained with adversarial training \cite{madry2017towards} against $\ell_1$, $\ell_2$, or $\ell_\infty$\ perturbations, respectively. Next, we compare against variants of adversarial training: the MAX, AVG and MSD approaches \cite{maini2020adversarial,tramer2019adversarial}. Finally, we compare against the BSC baseline. 

\myparagraph{Quality of Defense} Table \ref{table:mnist_def} summarizes our results on the MNIST dataset. The top half demonstrates that our proposed block-sparse approach improves upon state of the art adversarial training defenses against a union of attacks from $\ell_2$, $\ell_1$, and $\ell_\infty$ PGD attacks by about 15\% on average. The high accuracy of the Denoiser+CNN model also shows that $\Ds[\hat{i}] \csht[\hat{i}]$ is a good model of the denoised data. Surprisingly, even though the block-sparse classifier is not the strongest baseline for MNIST, as indicated by the relatively low clean accuracy of 94\%, we observe that it is much more robust to $\ell_p$ perturbations than the neural network models as the strength of the attack increases. 

\myparagraph{Performance on unseen test-time attacks} Our dictionary $\Da$ consists of $\{\ell_1, \ell_2, \ell_\infty\}$ PGD attacks, so the SBSAD must predict one of these three classes. However, we can evaluate our method on test-time attacks that are non-PGD $\ell_p$ attacks for $p \in \{1,2, \infty\}$. The second half of Table \ref{table:mnist_def} demonstrates the accuracy of our method on the $\ell_\infty$ Momentum Iterative Method (MIM) \cite{dong2018boosting}, the $\ell_2$ Carlini-Wagner (C-W) attack \cite{carlini_towards_2017}, and the $\ell_2$ Decoupled Direction and Norm (DDN) attack \cite{rony2019decoupling}. The block-sparse baseline performs remarkably well at denoising even though it does not model the attack structure. In principle, it does not make sense for our method to capture these attacks through $\Da$; however on average, we still observe a slight increase in accuracy by modelling some portion of the perturbation through the SBSC method. Perhaps more surprisingly, our method still has high attack classification accuracy, indicating that for the purposes of determining the attack family, the attacks can be well-approximated by a linear combination of PGD attacks.

\section{Conclusion}

In this paper, we studied the conditions under which we can reverse engineer adversarial attacks by determining the type of attack from a corrupted signal. We provided a structured block-sparse optimization approach to model not only the signal as a block-sparse combination of datapoints, but also the attack perturbation as a block-sparse combination of attacks. Under this optimization approach, we derived theoretical conditions under which recovery of the correct signal and attack type is feasible.  Finally, we experimentally verified the validity of the structured block-sparse optimization approach on the YaleB and MNIST datasets. We believe there are many directions to further study the properties of block-sparse classifiers, such as introducing non-linear embedding dictionaries.

\bibliography{adversarial,learning,recognition,sparse,vidal,vision}
\bibliographystyle{plain}



\newpage
\appendix
\onecolumn

\section{Theoretical Results}
We first give the proof of Proposition 5.1, which is based on the proofs of relevant results in subspace-sparse recovery (Theorem 2, \cite{Elhamifar:TPAMI13,You:arxiv15-SSR,You:ICML15}) and atomic representation-based recovery (Lemma 2, \cite{Wang:PAMI2017}.
\subsection{Proof of Proposition 5.1.}
\begin{proof}
($\implies$)

We first prove that \eqref{prop:nec_suf_rec} is a sufficient condition for recovering the correct class $i^*$ of the signal $\x$ and $(i^*j^*)$ of the attack $\deltab$. Let $\cs^\ast,\ca^\ast$ be optimal solutions of the problem
\begin{equation}
\begin{split}
    \{\cs^*,\ca^*\} \equiv \argmin_{\cs,\ca}{\|\cs\|_2+\|\ca\|_2} \\ 
    \st \x' = \Ds\cs + \Da\ca,
    \end{split}
    \label{prop:proof_orig_prop}
\end{equation}	
 The correct classes of the signal and the attack can be recovered when $\cs^\ast[i] = 0$ for $i\neq i^\ast$ and $\ca^\ast[i][j]=0$ for $i\neq i^*$ and $j\neq j^*$. We prove the sufficiency of condition of \eqref{prop:nec_suf_rec} for correct recovery of the classes of signal and attack by contradiction. Let us assume that there exist $\cs[i]\neq 0$ for $i\neq i^*$ and  $\ca^\ast[i][j]\neq 0$ for $i\neq^*,j\neq j^*$ and 
  define a vector $\tilde{\x}$ as,
\begin{equation}
\tilde{\x} = \x' - \Ds[i^*]\cs^*[i^*] + \Da[i^*][j^*]\ca[i^*][j^*] = \sum_{i\neq i^*}\Ds[i]\cs^\ast[i] + \sum_{i\in \Ical,j\neq j*}\Da[i][j]\ca^*[i][j].
\label{eq:proof_prop}
\end{equation}
%
Since $\x'\in \Scal^{\x}_{i^*}\oplus \Scal^{\deltab}_{i^*j^*}$, from \eqref{eq:proof_prop} we deduce that $\tilde{\x}$ will have a representation on $\Scal^{\x}_{i^*}\oplus \Scal^{\deltab}_{i^*j^*}$, which will be a feasible solution of the following block-sparse optimization problem.
\begin{equation}
\begin{split}
    \{\cshx,\cahx\} \equiv \argmin_{\cs,\ca}{\|\cs[i^\ast]\|_2+\|\ca[i^\ast][j^\ast]\|_2} \\ 
    \st \tilde{\x} = \Ds[i^\ast]\cs[i^\ast] + \Da[i^\ast][j^\ast]\ca[i^\ast][j^\ast],
    \end{split}
    \label{prop:proof_eq1}
\end{equation}	
where  $\csh,\cah$ are the correct-class minimum $\ell_1/\ell_2$ norm vectors supported on $\csh[i^*]$ and $\cah[i^*][j^*]$.
Moreover, from the \eqref{eq:proof_prop} we can also see that  $\tilde{\x}$ is also belong to the span of the union of subspaces of remaining blocks of the dictionaries and hence the following problem  
\begin{equation}
\begin{split}
   \{\cstx,\catx \} \equiv  \argmin_{\cs,\ca}\sum_{i\in \Ical \setminus \{i^{\ast}\} } \|\cs[i]\|_2  
    +\sum_{i\in \Ical , j\in \Jcal\setminus \{j^{\ast}\} }  \|\ca[i][j]\|_2   \\
   \st  \tilde{\x} = \sum_{i \in \Ical \setminus \{i^{\ast}\}}~\Ds[i]\cs[i]+\sum_{i\in \Ical, j\in \Jcal\setminus  \{j^{\ast}\} }~\Da[i][j]\ca[i][j] 
        \end{split}
        \label{prop:wrong_blk_prob}
\end{equation}
with $\Ical = \{1,2,\dots,r\}$ and $\Jcal = \{1,2,\dots,a\}$, will have feasible solutions. 
From \eqref{eq:proof_prop}, we can get,
\begin{equation}
\begin{split}
\x' &= \tilde{\x} + \Ds[i^*]\cs^*[i^*] + \Da[i^*][j^*]\ca^*[i^*][j^*] \\
 & = \Ds[i^*]\left(\cshx[i^*] + \cs^*[i^*]\right) + \Da[i^*][j^*]\left(\cahx[i^*][j^*] + \ca^*[i^*][j^*]\right)
\end{split}
\label{eq:prop_proof_eq2}
\end{equation}
From \eqref{eq:prop_proof_eq2} we can see that vectors the pair of vectors  $\cshx + \cs^*$ supported on the $i^*$th block and $\cahx + \ca^*$, supported on the $(i^*,j^*)$th block will be a feasible solution of  \eqref{prop:proof_orig_prop}.
We will have,
\begin{equation}
\begin{split}
  \|\cshx[i^*] + \cs^*[i^*]\|_2 +  \|\cahx[i^*][j^*] + \ca^*[i^*][j^*]\|_2 & \leq \\
  \|\cshx[i^*]\|_2 +  \|\cs^*[i^*]\|_2\|_2 + \|\cshx[i^*][j^*]\|_2 + \|\ca^*[i^*][j^*]\|_2  \\ <
   \|\cstx\|_{1,2} + \|\cs^*[i^*]\|_2 + \|\catx\|_{1,2} + \|\ca^*[i^*][j^*]\|_2 \leq \\
   \sum_{i\neq i^*}\|\cs^\ast[i]\|_2 + \|\cs^*[i^*]\|_2\|_2 + \sum_{i\in \Ical,j\neq j*}\|\ca^*[i][j]\|_2 + \|\|\ca^*[i^*][j^*]\|_2 =\|\cs^*\|_{1,2} + \|\ca^*\|_{1,2},
\end{split}
\label{prop:proof_ineq}
\end{equation}
where the second to the last inequality comes for the condition \eqref{prop:nec_suf_rec} of the Proposition. The last inequality in \eqref{prop:proof_ineq} appears due to optimality of $\cst,\cat$ in \eqref{prop:wrong_blk_prob} and the fact that a vector supported on the blocks of $\cs^*[i]$ for $i\neq i^*$ and $\ca^*[i][j]$ for $i\neq i^*$, $j\neq j^*$ is also a feasible solution of \eqref{prop:wrong_blk_prob}, yet not optimal. Hence we have  arrived at a contradiction since by optimality of $\cs^*,\ca^*$ the inequality $\|\cshx + \cs^*\|_{1,2} + \|\cahx + \ca^*\|_{1,2}<\|\cs^*\|_{1,2} + \|\ca^*\|_{1,2}$ can not be true. We thus proved that the sufficiency of the condition \eqref{prop:nec_suf_rec} for correct recovery of signal and attack classes.

($\Longleftarrow$) Let first define $\cshxx,\cahxx,\cstxx,\catxx$ as
\begin{equation}
\begin{split}
    \{\cshx,\cahx\} \equiv \argmin_{\cs,\ca}{\|\cs[i^\ast]\|_2+\|\ca[i^\ast][j^\ast]\|_2} \\ 
    \st \x' = \Ds[i^\ast]\cs[i^\ast] + \Da[i^\ast][j^\ast]\ca[i^\ast][j^\ast],
    \end{split}
    \label{prop:proof_eq2}
\end{equation}	
and the {\it wrong-class minimum $\ell_{1}/\ell_{2}$ norm vectors} $\cst,\cat$ as,
\begin{equation}
\begin{split}
   \{\cstxx,\catxx \} \equiv  \argmin_{\cs,\ca}\sum_{i\in \Ical \setminus \{i^{\ast}\} } \|\cs[i]\|_2  
    +\sum_{i\in \Ical , j\in \Jcal\setminus \{j^{\ast}\} }  \|\ca[i][j]\|_2   \\
   \st  \x' = \sum_{i \in \Ical \setminus \{i^{\ast}\}}~\Ds[i]\cs[i]+\sum_{i\in \Ical, j\in \Jcal\setminus  \{j^{\ast}\} }~\Da[i][j]\ca[i][j] 
        \end{split}
\label{prop:wrong_prob}
\end{equation}

Recall that the correct classes of the signal and the attack for an $\x'\in \Scal^\x_{i^*}\oplus \Scal^{\deltab}_{i^*j^*}$ can be recovered when the optimal $\cs^*,\ca^*$ are non-zero only at blocks $\cs^*[i^*]$ and $\ca^*[i^*][j^*]$. In that case, it also holds that $\|\cs^*\|_{1,2} + \|\ca^*\|_{1,2} = \|\cshxx\|_{1,2} + \|\cahxx\|_{1,2}$. We will show that if the correct classes of the signal and the attack can be recovered for $\x'$ then the condition \eqref{prop:nec_suf_rec} is true. 

For that we assume that the solution $\cstxx,\catxx$ is also feasible for problem \eqref{prop:proof_orig_prop} otherwise condition \eqref{prop:nec_suf_rec} is trivially satisfied since the RHS of \eqref{prop:nec_suf_rec} becomes $+\infty$.

Assume now that condition \eqref{prop:nec_suf_rec} is not true, i.e.,
\begin{equation}
\begin{split}
    \|\cshx\|_{1,2} + \|\cahx\|_{1,2} \geq \|\cstxx\|_{1,2} + \|\catxx\|_{1,2} 
\end{split}
\label{prop:ineq_cont}
\end{equation}
that will imply,
\begin{equation}
 \|\cs^*\|_{1,2} + \|\ca^*\|_{1,2} \geq \|\cstxx\|_{1,2} + \|\catxx\|_{1,2}
\end{equation}
and from optimality of $\cs^*,\ca^*$ and feasibility of $\cshxx,\cahxx$ at problem \eqref{prop:proof_orig_prop}, we will have that equality will hold, i.e.,
\begin{equation}
\|\cs^*\|_{1,2} + \|\ca^*\|_{1,2} = \|\cstxx\|_{1,2} + \|\catxx\|_{1,2}.
\end{equation}
The latter means that there will be an optimal solution $\{\cstxx,\catxx\}$ of \eqref{prop:proof_orig_prop} with non-zero blocks at indices corresponding to wrong classes of the signal and the attack when \eqref{prop:ineq_cont} holds true (i.e. condition \eqref{prop:nec_suf_rec} is false) which contradicts the initial assumption for the correct recovery of the classes of the signal and the attack. 
\end{proof}
Let  $\Dcal_{i^{\ast}j^{\ast}}$ be the set of atoms, which contains the columns of the blocks of dictionaries of the signal and the attack that correspond to the correct classes i.e., $[\Ds[i^{\ast}], \Da[i^{\ast}][j^{\ast}]]$. 
Recall from \eqref{def:relative_polar} that the relative polar set of $\pm \Dcal_{i^{\ast}j^{\ast}}$ induced by the $\ell_{1,2}$ norm is given as,
\begin{equation}
\mathcal{K}^{o}_{\ell_{1,2}}(\pm \Dcal_{i^{\ast}j^{\ast}})= \{ \v \in \mathrm{span}(\Scal_{i^{\ast}} \cup  \Scal_{i^{\ast}j^{\ast}}) : \frac{1}{\sqrt{m}}\|[\Ds[i^{\ast}],\Da[i^{\ast}][j^{\ast}]]^{\top}\v\|_{\infty,2}\leq 1 \}
\end{equation}
where $\mathrm{span}(\Scal_{i^{\ast}} \cup  \Scal_{i^{\ast}j^{\ast}})$ is the column-space of $[\Ds[i^{\ast}],\Da[i^{\ast}][j^{\ast}]]^{\top}]$.
Next we define the circumradius of a convex body.
\begin{definition}
(Circumradius) The circumradius of a convex body $\mathcal{P}$ denoted as $R(\mathcal{P})$ is defined as the radius of the smallest euclidean ball containing $\mathcal{P}$.
\end{definition}
In our case, we will use the circumradius of the convex hull of the set $\mathcal{K}^{o}_{\ell_{1,2}}(\pm \Dcal_{i^{\ast}j^{\ast}})$, denoted as $R(\mathcal{K}^{o}_{\ell_{1,2}}(\pm \Dcal_{i^{\ast}j^{\ast}}))$.

Lemma \ref{lem:circum_covering_rad} shows  the relationship between the covering radius of a set induced be $\ell_{1,2}$ norm and corresponding circumradius of its relative polar set.
\begin{lemma}
It holds that $\cos\left(\gamma_{1,2}(\pm \Dcal)\right) = \frac{1}{ R( \Kcal^{o}_{\ell_{1,2}}(\pm \Dcal))}$.
\label{lem:circum_covering_rad}
\end{lemma}
\begin{proof}
Our proof is based on that of the relevant result for the sparse recovery given in \cite{Robinson:Arxiv2019}.
The covering radius  $\gamma_{1,2}(\Dcal)$ is defined as 
\begin{equation}
\gamma_{1,2}(\pm \Dcal) = \sup\{\theta_{1,2}(\v,\pm \Dcal), \v\in \Sb^{n-1}\cap \mathrm{span}(\Dcal) \} = \sup_{\v\in \Sb^{n-1}\cap \mathrm{span}(\Dcal)}\{\cos^{-1}\left(\frac{1}{\sqrt{m}}\|\D^{\top}\v\|_{\infty,2}\right)\}
\end{equation}
Getting the cosine of $\gamma_{1,2}(\pm \Dcal)$ we have,
\begin{equation}
\cos\left( \gamma_{1,2}(\pm \Dcal) \right) = \cos\left( \sup_{\v\in \Sb^{n-1}\cap \mathrm{span}(\Dcal)}\{\cos^{-1}\left(\frac{1}{\sqrt{m}}\| \D^{\top}\v\|_{\infty,2}\right)\} \right) = \inf_{\v\in \Sb^{n-1}\cap \mathrm{span}(\Dcal)}~\frac{1}{\sqrt{m}}\|\D^{\top}\v\|_{\infty,2}
\label{lem_1_eq1}
\end{equation}
The circumradius $R(\mathcal{K}^{o}_{\ell_{1,2}}(\pm \Dcal))$ of the  relative polar set  $\mathcal{K}^{o}_{\ell_{1,2}}(\pm \Dcal)$ is given by,
\begin{equation}
R(\mathcal{K}^{o}_{\ell_{1,2}}(\pm \Dcal)) = \sup\{\|\v\|_{2}:\frac{1}{\sqrt{m}} \|\D^{\top}\v\|_{\infty,2}\leq 1, \v\in \mathrm{span}(\Dcal)\}
\label{lem_1_eq2}
\end{equation}
We want to prove that,
\begin{equation}
 \inf_{\v\in \Sb^{n-1}\cap \mathrm{span}(\Dcal)}~\frac{1}{\sqrt{m}}\|\D^{\top}\v\|_{\infty,2} =  \frac{1}{\sup\{\|\v\|_{2}: \frac{1}{\sqrt{m}}\| \D^{\top}\v\|_{\infty,2}\leq 1, \v\in \mathrm{span}(\Dcal)\}}
 \label{lem_1_eq3}
\end{equation}
Let   $\w^{\ast}$ and $\v^{\ast}$  be optimal solutions of the optimization problems appearing at the LHS and RHS of \eqref{lem_1_eq3}, respectively. Let us now define $\bar{\v} = \frac{\sqrt{m}\w^{\ast}}{\|\D^{\top}\w^{\ast}\|_{\infty,2}}$ and $\bar{\w} = \frac{\v^*}{\|\v^{\ast}\|_{2}}$. We have that $\|\w^{\ast}\|_{2}=1$ and $\bar{\v}$ satisfies the constraints appearing the optimization problem at the RHS of  \eqref{lem_1_eq3} i.e., $\frac{1}{\sqrt{m}}\|\D^{\top}\bar{\v}\|_{\infty,2}\leq 1$ and $\bar{\v}\in \mathrm{span}(\Dcal)$. 
Hence, we will have,
\begin{equation}
 \|\bar{\v}\|_{2} =  \frac{\sqrt{m}\|\w^{\ast}\|_{2}}{\|\D^{\top}\w^{\ast}\|_{\infty,2}} = \frac{\sqrt{m}}{\|\D^{\top}\w^{\ast}\|_{\infty,2}} \leq \|\v^{\ast}\|_{2}
 \label{lem_1_eq4}
\end{equation}
where the last inequality arises by the fact that $\bar{\v}$ is a feasible but not optimal solution of the problem at the RHS of \eqref{lem_1_eq3}.
Moreover, for $\bar{\w} = \frac{\v^{\ast}}{\|\v^{\ast}\|_{2}}$ we have that $\bar{\w}\in \Sb^{n-1}$ and $\bar{\w}\in \mathrm{span}(\Dcal)$. Therefore, $\bar{\w}$ satisfies the constraints and it will be a feasible solution of of the optimization problem at the LHS of \eqref{lem_1_eq3}. From that we can deduce that
\begin{equation}
\frac{1}{\sqrt{m}}\|\D^{\top}\bar{\w}\|_{\infty,2} = \frac{1}{\sqrt{m}}\frac{\|\D^{\top}\v^{\ast}\|_{\infty,2}}{\|\v^{\ast}\|_{2}} \leq \frac{1}{\|\v^{\ast}\|_{2}}
\end{equation}
From optimality of $\w^*$ at the LHS of \eqref{lem_1_eq3} we will have
\begin{equation}
\frac{1}{\|\v^{\ast}\|_{2}} \geq \frac{1}{\sqrt{m}} \|\D^{\top}{\w^{\ast}}\|_{\infty,2} 
 \rightarrow
\frac{\sqrt{m}}{  \|\D^{\top}{\w^{\ast}}\|_{\infty,2}} \geq \|\v^{\ast}\|_{2}
\label{lem_1_eq5}
\end{equation}

By combining \eqref{lem_1_eq4} and \eqref{lem_1_eq5} we get the result.
 \end{proof}

\subsection{Proof of Theorem \ref{the:prc_cond}}
Without loss of generality for the proofs of theorems \ref{the:prc_cond} and \ref{the:drc_cond} we scale the dictionaries $\Ds,\Da$ by $\frac{1}{\sqrt{m}}$, where $m$ is the size of the blocks. The primal problem denoted as $P(\frac{1}{\sqrt{m}}\Ds,\frac{1}{\sqrt{m}}\Da,\x')$ is given as,
\begin{equation}
P(\frac{1}{\sqrt{m}}\Ds,\frac{1}{\sqrt{m}}\Da,\x') \coloneqq \argmin \|\cs\|_{1,2} + \|\ca\|_{1,2}~~~\st \x' = \frac{1}{\sqrt{m}}\Ds\cs + \frac{1}{\sqrt{m}}\Da\ca
\label{eq:primal}
\end{equation}
and the dual of \eqref{eq:primal},
\begin{equation}
D(\frac{1}{\sqrt{m}}\Ds,\frac{1}{\sqrt{m}}\Da,\x') \coloneqq \argmax \langle \w, \x' \rangle ~~~\st  \|[\frac{1}{\sqrt{m}}\Ds,\frac{1}{\sqrt{m}}\Da]^{\top}\w\|_{\infty,2}\leq 1
\label{eq:dual}
\end{equation}
where $\w$ is the dual variable.
\begin{proof}
We will prove the theorem by showing that the condition 
\begin{equation}
\gamma_{1,2}(\pm \Dcal_{i^{\ast}j^{\ast}})  < \theta_{1,2}(\mathcal{S}_{i^{\ast}}\cup\mathcal{S}_{i^{\ast} j^{\ast}}, \pm \Dcal_{i^{\ast}j^{\ast}}^{-})
\end{equation}
implies the necessary and sufficient condition of Proposition 1, i.e., 
\begin{equation}
\|\csh\|_{1,2} + \|\cah\|_{1,2} < \|\cst\|_{1,2} + \|\cat\|_{1,2}
\label{theor_1:nsfcond}
\end{equation}

Let  us focus on $\frac{1}{\sqrt{m}}\Ds[i^{\ast}],\frac{1}{\sqrt{m}}\Da[i^{\ast}][j^{\ast}]$ and denote as $p(\frac{1}{\sqrt{m}}\Ds[i^{\ast}],\frac{1}{\sqrt{m}}\Da[i^{\ast}][j^{\ast}],\x'), d(\frac{1}{\sqrt{m}}\Ds[i^{\ast}],\frac{1}{\sqrt{m}}\Da[i^{\ast}][j^{\ast}],\x')$  the values of the objective functions of the primal and dual problems, respectively. Due to convexity, strong duality holds, hence we have,
\begin{equation}
p(\frac{1}{\sqrt{m}}\Ds[i^{\ast}],\frac{1}{\sqrt{m}}\Da[i^{\ast}][j^{\ast}]\x') = d(\frac{1}{\sqrt{m}}\Ds[i^{\ast}],\frac{1}{\sqrt{m}}\Da[i^{\ast}][j^{\ast}],\x') = \langle \w,\x' \rangle
\label{eq:theor_1}
\end{equation}
Let us now decompose the dual variable $\w\in \Rb^{n}$ as $\w = \w^{\perp} + \w^{\parallel}$, where $\w^{\parallel}\in \Scal_{i^{\ast}}\cup \Scal_{i^{\ast}j^{\ast}}$ and $\w^{\perp} \perp \w^{\parallel}$. For \eqref{eq:theor_1} we have,
\begin{equation}
\begin{split}
p(\frac{1}{\sqrt{m}}\Ds[i^{\ast}],\frac{1}{\sqrt{m}}\Da[i^{\ast}][j^{\ast}],\x') = d(\frac{1}{\sqrt{m}}\Ds[i^{\ast}],\frac{1}{\sqrt{m}}\Da[i^{\ast}][j^{\ast}],\x') = \\
\langle \w,\x' \rangle = \langle \w^{\parallel},\x'\rangle \leq \|\w^{\parallel}\|_{2}\|\x'\|_{2} \leq \|\x'\|_{2}\frac{1}{\cos(\gamma_{1,2}(\pm \Dcal_{i^{\ast}j^{\ast}}))}
\end{split}
\label{eq:theor_1b}
\end{equation}
where the last inequality follows from Lemma \eqref{lem:circum_covering_rad}, by taking into account that $\w^{\parallel}$ a) belongs to the dual polar set $\Kcal^{o}_{\ell_{1,2}}(\Dcal_{i^{\ast}j^{\ast}})$ b) $\w^{\parallel}\in\Scal_{i^{\ast}}\cup\Scal_{i^{\ast}j^{\ast}}$ and hence is a feasible solution of the optimization problem at the RHS of \eqref{lem_1_eq3}.

Let us now focus the primal problem,
\begin{equation}
p(\frac{1}{\sqrt{m}}\Ds^{-},\frac{1}{\sqrt{m}}\Da^{-}\x') = \min_{\cs,\cs} \|\cs\|_{1,2} + \|\ca\|_{1,2} \st \x' = \frac{1}{\sqrt{m}}\Ds^{-}\cs + \frac{1}{\sqrt{m}}\Da^{-}\ca
\end{equation}
and assume that there exist solutions $\cs^{\ast},\ca^{\ast}\in P(\frac{1}{\sqrt{m}}\Ds^{-},\frac{1}{\sqrt{m}}\Da^{-}\x'))$ such that $\x' = \Ds^{-}\cs^{\ast} + \Da^{-}\ca^{\ast}$. We will have,
\begin{equation}
\begin{split}
\|\x'\|^{2}_{2} = \x'^{\top}\left(\frac{1}{\sqrt{m}}\Ds^{-}\cs^{\ast} + \frac{1}{\sqrt{m}}\Da^{-}\ca^{\ast}  \right) \\
\leq \|\Ds^{-,\top}\frac{\x'}{\|\x'\|_{2}}\|_{\infty,2} \|\x'\|_{2}\|\cs^{\ast}\|_{1,2} + \|\Da^{-,\top}\frac{\x'}{\|\x'\|_{2}}\|_{\infty,2}\|\x'\|_{2}\|\ca^{\ast}\|_{1,2} \\
\leq \cos(\theta_{1,2}(\frac{\x'}{\|\x'\|_{2}},\Dcal^{-}_{i^{\ast},j^{\ast}}))\|\x'\|_{2}\underbrace{\left(\|\cs\|_{1,2} + \|\ca\|_{1,2} \right)}_{p(\frac{1}{\sqrt{m}}\Ds^{-},\frac{1}{\sqrt{m}}\Da^{-}\x') }\\
\rightarrow p(\frac{1}{\sqrt{m}}\Ds^{-},\frac{1}{\sqrt{m}}\Da^{-}\x') \geq \frac{\|\x'\|_{2}}{\cos(\theta_{1,2}(\frac{\x'}{\|\x'\|_{2}},\Dcal^{-}_{i^{\ast},j^{\ast}}))}
\end{split}
\label{eq:theor_1c}
\end{equation}
By combining \eqref{eq:theor_1b} with \eqref{eq:theor_1c} we get,
\begin{equation}
\begin{split}
\|\x'\|_{2}\frac{1}{\cos(\gamma_{1,2}(\pm \Dcal_{i^{\ast}j^{\ast}}))} < \frac{\|\x'\|_{2}}{\cos(\theta_{1,2}(\frac{\x'}{\|\x'\|_{2}},\Dcal^{-}_{i^{\ast},j^{\ast}}))} \rightarrow \\
\gamma_{1,2}(\pm \Dcal_{i^{\ast}j^{\ast}})) < \theta_{1,2}(\frac{\x'}{\|\x'\|_{2}},\Dcal^{-}_{i^{\ast},j^{\ast}}))
\end{split}
\end{equation}
and hence the last inequality is a sufficient condition for \eqref{theor_1:nsfcond}.
\end{proof}

\subsection{Proof of Theorem \ref{the:drc_cond}}
We first prove prove the following Lemma.

\begin{lemma}
If the Dual Recovery Condition holds i.e., $\gamma_{1,2}(\Dcal_{i^{\ast}j^{\ast}}) <  \theta_{1,2}\left(\mathcal{A}\left(\Dcal_{i^{\ast}j^{\ast}}\right), \pm \Dcal_{i^{\ast}j^{\ast}}^{-}\right)$ then $\forall\v\in \mathcal{A}(\Dcal_{i^{\ast}j^{\ast}})$ it holds
$\frac{1}{\sqrt{m}}\|[\D_{s}^{-,\top}, \D_{a}^{-,\top}]\v\|_{\infty,2} < 1$.
\label{lem:the_drc_1}
\end{lemma}
\begin{proof}
We have that $\forall \v\in \mathcal{A}(\Dcal_{i^{\ast}j^{\ast}})$ it holds $\frac{1}{\sqrt{m}}\|[\D_{s}[i^{\ast}],\D_{a}[i^{\ast}][j^{\ast}]]^{\top}\v\|_{\infty,2}\leq 1$. Hence, due to Lemma \ref{lem:circum_covering_rad} we have that $\|\v\|_{2} \leq \frac{1}{\cos(\gamma_{1,2}(\Dcal_{i^{\ast}j^{\ast}})}$. We will have,
\begin{equation}
\frac{1}{\sqrt{m}}\|[\D_{s}^{-}, \D_{a}^{-}]^{\top}\v\|_{\infty,2} = \frac{1}{\sqrt{m}}\|[\D_{s}^{-}, \D_{a}^{-}]^{\top}\frac{\v}{\|\v\|_{2}}\|_{\infty,2}\|\v\|_{2} \leq \frac{\cos(\theta_{1,2}(\v,\Dcal^{-}_{i^{\ast}j^{\ast}}))}{\cos(\gamma_{1,2}(\Dcal_{i^{\ast}j^{\ast}}))} < 1
\end{equation}
\end{proof}
Next we will prove the following Lemma,
\begin{lemma}
If $\frac{1}{\sqrt{m}}\|[\D_{s}^{-,\top}, \D_{a}^{-,\top}]\v\|_{\infty,2} < 1$ $\forall \v\in \mathcal{A}(\Dcal_{i^{\ast}j^{\ast}})$ then the necessary and sufficient condition for successful recovery of the correct class of the signal and the attack given in \eqref{prop:nec_suf_rec} i.e, 
$p(\frac{1}{\sqrt{m}}\Ds[i^{\ast}],\frac{1}{\sqrt{m}}\Da[i^{\ast}][j^{\ast}],\x') < p(\frac{1}{\sqrt{m}}\Ds^{-},\frac{1}{\sqrt{m}}\Da^{-},\x') $ holds.
\label{lem:the_drc_2}
\end{lemma}
\begin{proof}
Let us define the following constrained optimization problem,
\begin{equation}
\max \langle \x',\w \rangle ~~\st \frac{1}{\sqrt{m}}\|[\D_{s}[i^{\ast}],\D_{a}[i^{\ast}][j^{\ast}]]^{\top},\w\|_{\infty,2}\leq 1, \w\in \mathrm{span}(\Dcal_{i^{\ast}j^{\ast}})
\label{the:drc_1}
\end{equation}
Using standard convex optimization arguments we can deduce that the optimal solution of the above problem $\w$  will be an extreme point of the convex set defined by $\{\w :  \frac{1}{\sqrt{m}}\|[\D_{s}[i^{\ast}],\D_{a}[i^{\ast}][j^{\ast}]]^{\top}\w\|_{\infty,2}\leq 1, \w\in \mathrm{span}(\Dcal_{i^{\ast}j^{\ast}}) \}$ hence 
$\w$ will belong to the set of dual points $\mathcal{A}(\Dcal_{i^{\ast}j^{\ast}})$. Let us now state the following problem,
\begin{equation}
\max \langle \x',\w \rangle ~~\st \frac{1}{\sqrt{m}}\|[\D_{s}[i^{\ast}],\D_{a}[i^{\ast}][j^{\ast}]]^{\top}\w\|_{\infty,2}\leq 1, 
\label{the:drc_2}
\end{equation}
Note that \eqref{the:drc_2} does not constrain $\w$ to belong in $\mathrm{span}(\Dcal_{i^{\ast}j^{\ast}})$. As a result, there might be optimal solutions $\w$ not in $\mathrm{span}(\Dcal_{i^{\ast}j^{\ast}})$. However, we can deduce that there will always exist a 
$\w \ \mathrm{span}(\Dcal_{i^{\ast}j^{\ast}})$ that will be an optimal solution and a dual point. This can be deduced if we express a candidate solution $\w^\ast$ as $\w^\ast = \w^{\perp} + \w^{\parallel}$ where $\w^{\parallel}\in \mathrm{span}(\Dcal_{i^{\ast}j^{\ast}})$. 

Let us now assume that there exists a $\{\cs,\ca\} \in  P(\frac{1}{\sqrt{m}}\Ds^{-},\frac{1}{\sqrt{m}}\Da^{-},\x')$. We will have $\x' = \frac{1}{\sqrt{m}}\Ds^{-} \cs + \frac{1}{\sqrt{m}}\Da^{-}\ca$. On the other hand, there will be $\w^{\ast}\in \mathcal{A}(\Dcal_{i^{\ast}j^{\ast}})$ that will be a dual optimal solution of $D(\frac{1}{\sqrt{m}}\Ds[i^{\ast}],\frac{1}{\sqrt{m}}\Da[i^{\ast}][j^{\ast}],\x')$ i.e.,  
\begin{equation}
\begin{split}
p(\frac{1}{\sqrt{m}}\Ds[i^{\ast}],\frac{1}{\sqrt{m}}\Da[i^{\ast}][j^{\ast}],\x' ) = d(\frac{1}{\sqrt{m}}\Ds[i^\ast],\frac{1}{\sqrt{m}}\Da[i^\ast][j^\ast],\x') = \langle \w^{\ast}, \x'\rangle = \langle \w^{\ast}, \frac{1}{\sqrt{m}}(\Ds^{-} \cs + \Da^{-}\ca)\rangle \leq \\
 \frac{1}{\sqrt{m}}\|[\Ds^{-},\Da^{-}]^{\top}\w^{\ast}\|_{\infty,2}\left(\|\cs\|_{1,2} + \|\ca\|_{1,2}\right) < p(\frac{1}{\sqrt{m}}\Ds^{-},\frac{1}{\sqrt{m}}\Da^{-},\x' )
\end{split}
\end{equation}
\end{proof}

Theorem \ref{the:drc_cond} is proved by combining Lemmas \ref{lem:the_drc_1} and \ref{lem:the_drc_2}.

\section{Derivation of the Active Set Homotopy Algorithm and Algorithm Details}

\begin{algorithm}[h]
\caption{Active Set Homotopy Algorithm}
\label{alg:active_set}
\begin{algorithmic}
\STATE{Results:$\cah,\csh$ }
\STATE{ $\mathrm{Initialize:} \; \csht^0,\caht^0 \gets \mathbf{0}, \mathbf{0}, T_s^0, T_a^0 \gets \emptyset, \emptyset, k \gets 1$} 
\STATE{$\mathrm{Set:} \; \gamma \in (0, 1)$} 
\WHILE{$T_s^{k + 1} \not\subseteq T_s^k$ and $T_a^{k + 1} \not\subseteq T_a^k$} 
  \STATE{$\o^k \gets \x' - \Ds[T_s^k] \csht^k[T_s^k] - \Da[T_a^k] \caht^k [T_a^k]$} 
   \STATE{ $\lambda_s^k \gets \gamma \cdot \max_i \norm{\Ds[i]^T \o^k}{2}$}
    \STATE{$\lambda_a^k \gets \gamma \max_{i, j} \norm{\Da[i][j]^T \o^k}{2}$} 
    \STATE{$\hat{i}^k \gets \argmax_{i} \norm{\Da[i]^T \o^k}{2}$}
    \STATE{$\hat{j}^k \gets \argmax_j \max_i \norm{\Da[i][j]^T \o^k}{2}$} 
   \STATE{Add $\hat{i}^k$ and $\hat{j}^k$ to $T_s^k$ and $T_a^k$ respectively.}
    \STATE{Solve problem \eqref{eq:reg_objective} with any solver using $\Ds[T_s^k], \Da[T_a^k], \lambda_s^k, \lambda_a^k$ and compute $\csht^{k + 1}$ and $\caht^{k + 1}$.} 
   \STATE{$k \gets k + 1$}
    \ENDWHILE
\end{algorithmic}
\label{alg:activeset_homotopy}
\end{algorithm}
 
Consider the optimization problem in Equation \eqref{eq:reg_objective}. We denote the objective as $L(\x', \Ds, \Da, \lambda_s, \lambda_a)$. We can write the optimality conditions with respect to $\cs$ and $\ca$ for this problem. For any block $i$ of $\Ds$, the optimality conditions with respect to $\cs$ are:

\begin{align}
    \label{eq:cs_opt1} \Ds[i]^T(\x' - \Ds \cs^* - \Da \ca^*)  = \lambda_s \frac{\cs^*}{\norm{\cs^*}{2}} \qquad &\text {if } \cs^*[i] \neq 0 \\
    \label{eq:cs_opt2} \norm{\Ds[i]^T(\x' - \Ds \cs^* - \Da \ca^*)}{2} \leq \lambda_s \qquad &\text {if } \cs^*[i] = 0
\end{align}

For any block $(i, j)$ of $\Da$, the optimality conditions with respect to $\ca$ are:

\begin{align}
    \label{eq:ca_opt1} \Da[i][j]^T(\x' - \Ds \cs^* - \Da \ca^*)  = \lambda_2 \frac{\ca^*}{\norm{\ca^*}{2}} \qquad &\text {if } \ca^*[i][j] \neq 0 \\
    \label{eq:ca_opt2} \norm{\Da[i][j]^T(\x' - \Ds \cs^* - \Da \ca^*)}{2} \leq \lambda_2 \qquad &\text {if } \ca^*[i][j] = 0
\end{align}

First, we derive a value of $\lambda_s$ and $\lambda_s$ such that the optimal $\cs$ and $\ca$ are the all-zero vectors.

\begin{lemma}
    Let $\lambda_s \geq \norm{\Ds^T \x'}{\infty, 2} = \sup_i \norm{\Ds[i]^T \x'}{2}$ and $\lambda_a \geq \norm{\Da^T \x'}{\infty, 2} = \sup_{i, j} \norm{\Da[i][j]^T \x'}{2}$. Then, the values of $\cs^*$ and $\ca^*$ that minimize $L(x', \Ds, \Da, \lambda_s, \lambda_a))$ are the all-zero vectors. 
\end{lemma}

\begin{proof}
    We begin with the proof of showing that $\lambda_s = \norm{\Ds^T \x'}{\infty, 2}$ is sufficient so that $\cs^*$ is the all-zeroes vector. Looking at Equation \eqref{eq:cs_opt2}, we see that for a block of $\cs^*$ to be $0$, a sufficient condition is that the norm of the gradient of the fitting term of the objective is less than $\lambda_s$. This immediately gives that if for all blocks $i$,
    
    \begin{equation}
        \norm{\Ds[i]^T(\x' - \Ds \cs^* - \Da \ca^*)}{2} \leq \lambda_s
    \end{equation}
    
    then the optimal $\cs^*$ must be $0$ based on the optimality conditions. For simplicity in the proof, we will assume that the sufficient condition for $\ca$ to be the zero vector holds, which will be shown after. This implies that if $\lambda_s \geq \norm{\Ds^T \x'}{\infty, 2} = \sup_i \norm{\Ds[i]^T \x'}{2}$, then the $\cs^*$ that minimizes $L(\x', \Ds, \Da, \lambda_s, \lambda_a))$ is the all-zeros vector. The same argument applies for $\lambda_a$, for which we have that if $\lambda_a \geq \norm{\Da^T \x'}{\infty, 2} = \sup_{i, j} \norm{\Da[i][j]^T \x'}{2}$, then $\ca^*$ is the all-zeros vector. Jointly fixing both $\lambda_s$ and $\lambda_a$, we have a sufficient condition for $\cs^*$ and $\ca^*$ being zero. 
\end{proof}

The proof strategy of the above lemma suggests that if we knew the value of $\cs^*$ and $\ca^*$, then we can find a value of $\lambda_s$ and $\lambda_a$ such that minimizing $L(\x', \Ds, \Da, \lambda_s, \lambda_a))$ yields $\cs^*$ and $\ca^*$; however, obviously, we do not know the value of $\cs^*$ and $\ca^*$. Namely, the value of the regularization parameters depends on the residual $\x' - \Ds \cs^* - \Da \ca^*$, which we denote as $\o^*$ or the \textit{oracle point}. The homotopy algorithm for solving LASSO $\ell_1$ minimization problems proceeds by starting from the all-zeros solution and calculating the decrease in $\lambda$ that results in one non-zero element added to the support of the optimal solution. This works because the optimal solution plotted as a function of the regularization parameter is piecewise linear. For the block-sparse optimization problem, also known as group-LASSO, it is well-known that the solution path is nonlinear \cite{Yau:Stats2017}. Thus, we use the natural heuristic of starting with the value of $\lambda_s$ and $\lambda_a$ that produces the all-zero vector, scaling the value by some hyperparameter $\gamma \in (0, 1)$, and estimating the oracle point $\o$ by solving $L(\x', \Ds, \Da, \gamma \lambda_s, \gamma \lambda_a)$. From $\o$, we can then again calculate a value of $\lambda_s$ and $\lambda_a$ and iterate. This alternating algorithm forms the basis of the active set homotopy algorithm. Since we begin from the all-zero vector and reduce $\lambda_s$ and $\lambda_a$, we can maintain an active set of non-zero coordinates and only solve subproblems restricted to these non-zero blocks for efficiency purposes. In Algorithm \ref{alg:active_set}, we see the full algorithm detailed. Note that we overload the notation $\D_s[T_s]$ and $\D_a[T_a]$ to denote the submatrices of $\D_s$ and $\D_a$ corresponding to the block indices in the sets $T_s$ and $T_a$.To solve the subproblems, we use the cvxpy package \cite{Diamond:JMLR2016, Agrawal:JCD2018} with the SCS solver run for a maximum of 50 iterations.


\section{Experimental Details}

\begin{table}[h]
\begin{center}
\begin{tabular}{cc}
\hline
Layer Type & Size \\
\hline
Convolution + ReLU & $3 \times 3 \times 32$ \\
Convolution + ReLU & $3 \times 3 \times 32$ \\
Max Pooling & $2 \times 2$ \\
Convolution + ReLU & $3 \times 3 \times 64$ \\
Convolution + ReLU & $3 \times 3 \times 64$ \\
Max Pooling & $2 \times 2$ \\
Fully Connected + ReLU & $200$ \\
Fully Connected + ReLU & $200$ \\
Fully Connected + ReLU & $10$ \\
\hline
\end{tabular}
\end{center}
\caption{Network Architecture for the MNIST dataset}
\label{table:archs}
\end{table}

\subsection{MNIST}

The network architecture for the MNIST dataset is given in Table \ref{table:archs}. The network on MNIST is trained using SGD for $50$ epochs with learning rate $0.1$, momentum $0.5$, and batch size $128$.

All PGD adversaries were generated using the Advertorch library. The $\ell_\infty$ PGD adversary ($\epsilon = 0.3$) used a step size $\alpha = 0.01$ and was run for $100$ iterations. The $\ell_2$ PGD adversary ($\epsilon = 2$) used a step size $\alpha = 0.1$ and was run for $200$ iterations. The $\ell_1$ PGD adversary ($\epsilon = 10$) used a step size $\alpha = 0.8$ and was run for $100$ iterations. These hyperparameters are identical to the hyperparameters for the adversarial training baselines, to enable a fair comparison. 

\subsection{YaleB}

For the YaleB dataset, we train a three-layer fully-connected network, where each hidden layer contains 256 neurons followed by a ReLU activation. We train this network using SGD for 75 epochs with learning rate $0.05$, momentum $0.5$, and batch size $128$. All PGD adversaries were generated using the Advertorch library. The $\ell_\infty$ PGD adversary ($\epsilon = 0.1$) used a step size $\alpha = 0.003$ and was run for $100$ iterations. The $\ell_2$ PGD adversary ($\epsilon = 5$) used a step size $\alpha = 0.02$ and was run for $200$ iterations. The $\ell_1$ PGD adversary ($\epsilon = 15$) used a step size $\alpha = 1.0$ and was run for $100$ iterations.

\begin{figure*}[t!]
\begin{center}
   \includegraphics[width=\linewidth]{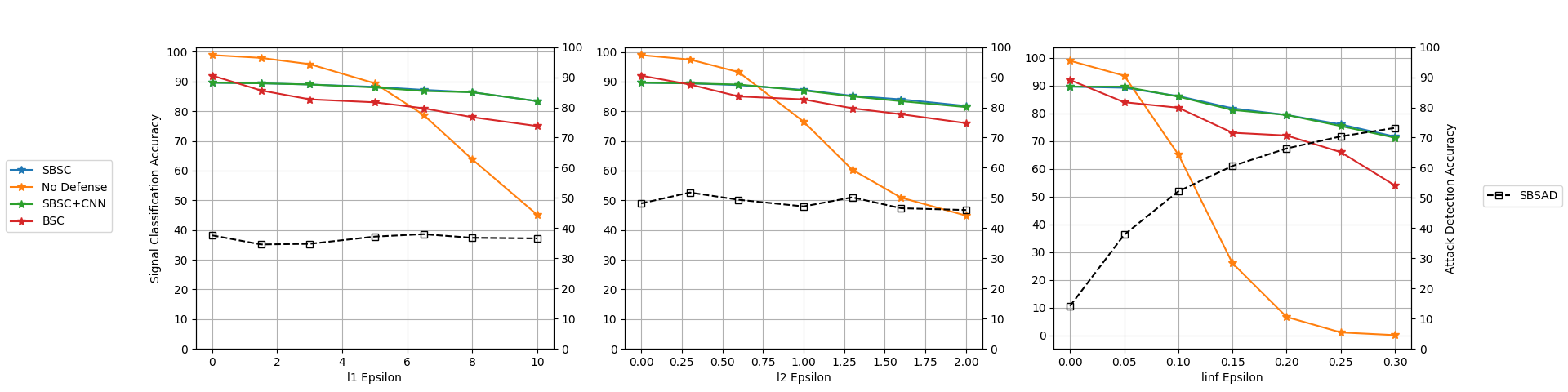}
\end{center}
   \caption{Results on the MNIST dataset with varying attack strength $\epsilon$. The SBSC and SBSC+CNN curves show the accuracy of the structured block-sparse classifier and denoiser at predicting the correct class. Note that these curves overlap almost fully, thus both are not clearly visible. BSC denotes the naive block-sparse baseline.  The SBSAD curve denotes the accuracy of the attack detector at predicting the correct attack type.  Best viewed in color.}
   \label{fig:lp_mnist}
\end{figure*}

\section{Extended Experiments on MNIST Dataset}

\myparagraph{Tradeoff between signal and attack detection} For $\ell_\infty$ attacks, the SBSC method does not serve as an effective defense compared to the baselines. One possible explanation for this phenomenon is the relationship between the SBSC and SBSAD methods, since both are predicted jointly from Algorithm 1. Specifically, in Figure \ref{fig:lp_mnist}, we see an explicit tradeoff between the choice of $\epsilon$ in terms of the signal classification and attack detection accuracy. As $\epsilon$ increases, we expect the attacks to be easier to distinguish among the family of attacks; however, as the noise increases, classifying the correct label becomes harder regardless of the accuracy of the predicted attack type. Note that in this figure, the dictionaries $\Ds$ and $\Da$ are kept fixed using the same $\epsilon$ values as in Table \ref{table:mnist_def}, and only the $\epsilon$ of the attacked test images is varied. Additionally, Figure \ref{fig:lp_mnist} demonstrates that explicitly modeling the perturbation and adding further structure to block-sparse optimization methods helps improve the accuracy of the block-sparse classifier, as our method outperforms the BSC method that only models $\x' = \Ds \cs$. As $\epsilon$ increases, the method remains robust to perturbations, while the accuracy of the undefended network continues to degrade.

\end{document}